\documentclass[letterpaper, 10 pt, conference]{ieeeconf}  

\usepackage{mathtools, amssymb, amsmath,bm,xcolor}
\usepackage[nolist]{acronym}
\usepackage[capitalize]{cleveref}
\usepackage{siunitx}
\usepackage{fp, tikz, pgfplots}
\usepgfplotslibrary{groupplots}
\usetikzlibrary{arrows}
\usetikzlibrary{shapes}
\usetikzlibrary{patterns}
\usetikzlibrary{decorations.pathmorphing}
\usetikzlibrary{decorations.pathreplacing}
\usetikzlibrary{decorations.shapes}
\usetikzlibrary{decorations.text}
\usetikzlibrary{shapes.misc}
\usetikzlibrary{decorations.markings}
\usetikzlibrary{arrows.meta}
\usetikzlibrary{backgrounds}

\usepackage{textcomp}
\usepackage{eso-pic}
\newcommand\copyrighttext{%
\footnotesize \textcopyright 2021 IEEE. Personal use of this material is permitted. Permission from IEEE must be obtained for all other uses, in any current or future media, including reprinting/republishing this material for advertising or promotional purposes, creating new collective works, for resale or redistribution to servers or lists, or reuse of any copyrighted component of this work in other works.}
\newcommand\copyrightnotice{%
\begin{tikzpicture}[remember picture,overlay]
\node[anchor=south,yshift=5pt] at (current page.south) {\fbox{\parbox{\dimexpr\textwidth-\fboxsep-\fboxrule\relax}{\copyrighttext}}};
\end{tikzpicture}%
}

\newcommand{\labelsize}[1]{\scalebox{0.8}{{\normalsize #1}}} 
\newcommand{\legendsize}[1]{\scalebox{0.7}{{\normalsize #1}}} 

\definecolor{red}{rgb}{0.5,0,0}
\definecolor{green}{rgb}{0.4660,0.6740,0.1880}
\definecolor{blue}{rgb}{0,0.5,0.75}
\definecolor{mycolor1}{rgb}{0.4940,0.1840,0.5560}%
\definecolor{mycolor2}{rgb}{1,0.5,0}%
\definecolor{mycolor3}{rgb}{0 0 0}%

\newtheorem{theorem}{Theorem}
\newtheorem{remark}{Remark}
\newtheorem{assumption}{Assumption}

\newcommand*{\mulm}[1]{\bm \mu_{i,m}^{(#1)}}
\newcommand*{\siglm}[1]{\Sigma_{i,m}^{(#1)}}
\newcommand*{\mul}[1]{\bm \mu_i^{(#1)}}
\newcommand*{\sigl}[1]{\Sigma_i^{(#1)}}

\newcommand*{\thet}{\bm \theta}
\newcommand*{\f}{\bm f}
\newcommand*{\h}{\bm h}
\newcommand*{\rn}{\bm r}
\newcommand*{\x}{\bm x}
\newcommand*{\dx}{\dot\x}
\newcommand*{\ddx}{\ddot\x}
\newcommand*{\p}{\bm p}
\newcommand*{\dpn}{\dot\p}
\newcommand*{\ddp}{\ddot\p}
\newcommand*{\q}{\bm q}
\newcommand*{\w}{\bm \omega}
\newcommand*{\dw}{\dot \w}

\newcommand*{\z}{\bm z}

\newcommand*{\Rr}[2]{{}^#2 \bm R_#1}
\newcommand*{\R}[1]{\bm R_#1}
\newcommand*{\M}{\bm M}
\newcommand*{\D}{\bm D}
\newcommand*{\K}{\bm K}
\newcommand*{\C}{\bm C}
\newcommand*{\J}{\bm J}
\newcommand*{\G}{\bm G}
\newcommand*{\I}{\bm I}
\newcommand*{\0}{\bm 0}
\newcommand*{\Scal}{\bm T(\w_i,\dw_i)}

\newcommand*{\dotm}[1]{[.#1]}

\newcommand*{\sumaj}{\sum_{j=1}^N}
\newcommand*{\pji}{{}^j\p_i}

\newcommand*{\Drij}[1]{\bm{r}_{#1}}
\newcommand*{\sref}[1]{Sec.~\ref{#1}}
\newcommand*{\fref}[1]{Fig.~\ref{#1}}
\newcommand*{\aref}[1]{Assumption~\ref{#1}}

\newcommand*{\bmat}[1]{\begin{bmatrix}
		#1
\end{bmatrix}}
\newcommand*{\Rs}{\mathbb{R}}
\newcommand*{\Gs}{\mathcal{G}}
\newcommand*{\Es}{\mathcal{E}}
\newcommand*{\Vs}{\mathcal{V}}

\newcommand*{\tran}{^{\mkern-1.5mu\mathsf{T}}}
\newcommand*{\estim}[1]{\widehat{#1}}
\newcommand*{\cov}[1]{\widetilde{#1}}
\newcommand*{\fett}[1]{\bm{#1}}
\newcommand*{\strich}{^{\prime}}

\newcommand*{\fpkt}[1]{\bm{\dot{#1}}}

\newcommand*{\orj}{{}^o\bm{r}_j}
\newcommand*{\ori}{{}^o\bm{r}_i}
\newcommand*{\Skew}[1]{S(\bm{#1})}
\newcommand{\Adj}{\bm{A}}

\newcommand*{\inv}{^{-1}}

\newcommand{\eqcolon}{\mathrel{\resizebox{\widthof{$\mathord{=}$}}{\height}{ $\!\!=\!\!\resizebox{1.2\width}{0.8\height}{\raisebox{0.23ex}{$\mathop{:}$}}\!\!$ }}}

\newcommand\mydots{\hbox to 0.75em{.\hss.\hss.}}

\IEEEoverridecommandlockouts                              

\title{\LARGE \bf
Distributed Bayesian Online Learning for Cooperative Manipulation
}

\author{Pablo Budde gen. Dohmann$^{1*}$, Armin Lederer$^{1*}$, Marcel Di\ss emond$^{1}$, Sandra Hirche$^{1}$
\thanks{$^{1}$ Chair of Information-oriented Control, Technical University of Munich, Germany, {\tt\small{pablo.dohmann, armin.lederer, marcel.dissemond, hirche}@tum.de}}%
 \thanks{$^{*}$These authors contributed equally.}
 \thanks{
This work was supported by the European Research Council (ERC) Consolidator Grant "Safe  data-driven  control  for human-centric systems (CO-MAN)" under grant agreement number 864686 and by the German Research Foundation (DFG) within the Joint Sino-German research project "Control and Optimization for Event-triggered Networked Autonomous Multi-agent Systems (COVEMAS)".
}}%

\begin{document}

\begin{acronym}
\acro{com}[COM]{center of mass}
\end{acronym}

\maketitle
\copyrightnotice
\AddToShipoutPicture{\begin{tikzpicture}[remember picture,overlay]
\node[anchor=north,yshift=-10pt,align=center] at (current page.north) {\footnotesize This is the accepted version of an article that has been published in 60th IEEE Conference on Decision and Control (CDC).\\\footnotesize The final published article can be found at DOI:10.1109/cdc45484.2021.9683772};
\end{tikzpicture}}

\thispagestyle{empty}
\pagestyle{empty}

\begin{abstract}

For tasks where the dynamics of multiple agents are physically coupled, the coordination between the individual agents becomes crucial, which requires exact knowledge of the interaction dynamics.
This problem is typically addressed using centralized estimators, which can negatively impact the flexibility and robustness of the overall system.
To overcome this shortcoming, we propose a novel distributed learning framework for the exemplary task of cooperative manipulation by applying Bayesian principles. 
Using only local state information each agent obtains an estimate of the object dynamics and grasp kinematics.
These local estimates are combined using dynamic average consensus. Due to the strong probabilistic foundation of the method, each estimate of the object dynamics and grasp kinematics is accompanied by a measure of uncertainty, which allows to guarantee a bounded prediction error with high probability. Moreover, the Bayesian principles directly allow iterative learning with constant complexity, such that the proposed learning method can be used online in real-time applications.
The effectiveness of the approach is demonstrated in a simulated cooperative manipulation task.

\end{abstract}

\setlength{\textfloatsep}{10pt}
\setlength{\floatsep}{10pt}


\section{INTRODUCTION}
Recent advances in communication networks allow for novel applications of distributed cooperative control approaches in multi-agent systems. Of special interest are tasks in which the dynamics of the individual agents are physically coupled, since a high degree of coordination is  required for a successful completion of the goal. 
For such tasks typically precise knowledge of the interaction dynamics is required, which is often unavailable in real world scenarios and estimation techniques are required. 
In this work, we present such an estimation framework and illustrate the steps by applying it to a cooperative manipulation scenario.
In cooperative manipulation the agents are physically coupled via the object and distributed control approaches have been under investigation recently~\cite{Dohmann.2020,Marino.2018.2}. For such a task it is critical to know the grasp kinematics and object dynamics in order to precisely manipulate the object and avoid internal stress, which could possibly damage the object~\cite{Erhart.2016}. The application domain varies greatly and ranges from construction and manufacturing to service robots or search and rescue scenarios.

This problem is usually addressed by applying centralized estimators~\cite{Cehajic.2017,Pierri.2020} that use the states of all robots involved in the manipulation task. This kind of estimation framework is prone to single-point failures and inflexible regarding changes in the number of agents involved in the manipulation task.
One way to overcome these shortcomings is to use a decentralized approach without any communication among the robot team. Each agent of the team calculates its own estimate of the unknown parameters. These local estimates can then be used for subsequent tasks \cite{Marino.2017}. While eliminating the requirements on the communication among the agents and introducing more flexibility, purely decentralized approaches omit the possibility of improving their local estimates by sharing those over a communication network. In order to combine the benefits of centralized and decentralized approaches, distributed estimation frameworks can be used by allowing for some communication. This can be done by locally calculating the estimates of the unknown parameters in a decentralized manner and iteratively combining them by some sort of consensus algorithm to achieve the convergence on a globally common estimate \cite{Marino.2018.2,Franchi.2015}. Thus, the communication overhead is minimized compared to the overhead imposed by a centralized approach while maintaining the flexibility and robustness of a decentralized approach.

None of the above mentioned approaches considers information about the uncertainty of estimates, i.e. they are all equally weighted during aggregation.
While this is a suitable approach when estimating parameters off-line using optimized input signals, it can significantly deteriorate the estimation performance in on-line learning scenarios with noisy measurements where input signals might not sufficiently excite all agents. Therefore, controllers relying on the estimated parameters permanently require a high robustness, since they have no information about the precision of the estimated parameters.
It is well-known from centralized systems that probabilistic approaches relying on Bayesian learning methods allow to mitigate these issues, as they explicitly represent the uncertainty of estimates \cite{Bishop2006}. This allows an aggregation of learned parameters weighted by their uncertainty \cite{Cao2014, Deisenroth.2015}, such that poor local estimates are oppressed in the overall aggregation. Moreover, a probabilistic representation exhibits many advantages in control, e.g., allowing to adapt control parameters to the uncertainty for ensuring stability \cite{Beckers2019}, achieving cautiousness in control \cite{Hewing2020}, or steering the system to regions with informative data~\cite{Capone2020c}.\looseness=-1

In this work, we propose a Bayesian distributed learning framework for estimating the object dynamics and grasp kinematics. By applying Bayesian inference to obtain the local estimates, each estimated parameter is accompanied by a measure of uncertainty. Moreover, the use of Bayesian principles readily allows iterative learning, such that the estimation framework can be used online in real-time applications. The estimated parameters and their uncertainties are propagated through a communication network using dynamic average consensus, such that an uncertainty-aware, distributed aggregation of the local estimates is achieved. The efficacy of the proposed framework is demonstrated in a simulation of a cooperative manipulation task.


The remainder of this work is structured as follows. The exact problem is formulated in \sref{sec:problem}. The theory substantiating the main contributions of this work is given in \sref{sec:main} and evaluated in numerical simulations in \sref{sec:evaluation}. \sref{sec:conclusion} briefly concludes the proposed contributions and presents future work.
\textbf{Notation:}
The matrices $\0_n,\I_n$ and $\bm 1_n$ denote the $n-$dimensional identity and zero matrix and vector with a $1$ in each element, respectively. The dimension $n$ is omitted if it is clear from context. The matrix $\Skew{\cdot}$ denotes the skew-symmetric matrix such that $\Skew{a}\bm b = \bm a\times \bm b, \ \forall \bm a,\bm b\in\Rs^3$.
A superscript $A^{ij}$ for a matrix $\bm A$ denotes the element of the matrix in the $i$th row and $j$th coloumn.
Scalar operators applied to vectors and matrices denote an element-wise operation.\looseness=-1
\section{PROBLEM STATEMENT}
\label{sec:problem}
In this work we consider $N$ manipulators rigidly grasping and manipulating a common rigid object. 
Each agent $i$ has a body-fixed frame~$\left\lbrace \text{i} \right\rbrace $ attached to its tool center point. The object frame at the object's \ac{com} is denoted as $\left\lbrace \text{o} \right\rbrace $ and $\left\lbrace \text{w} \right\rbrace $ is the inertial world frame. If not stated otherwise by a leading superscript~${}^i(\cdot)$, all states are given with respect to the world frame~$\left\lbrace w \right\rbrace $. 
The states $\x_i,\dx_i,\ddx_i$ denote the pose, velocity and acceleration of the $i$th endeffector, respectively. More precisely, the pose $\x_i = \bmat{\p_i\tran,\q_i\tran}\tran$ consists of the position $\p_i \in \Rs^3$ and the unit quaternion $\q_i = \bmat{\eta_i\tran & \bm \epsilon_i\tran}\tran \in \text{Spin}(3)$, with a scalar real part $\eta_i$ and imaginary vector part $\bm \epsilon_i$.
With a slight abuse of notation the velocity is defined as $\dx_i = \bmat{\dpn_i\tran & \w_i\tran}\tran$ with the angular velocity $\bmat{0 & \w_i\tran}\tran = 2 \frac{d}{dt} \q_i*\tilde{\q}_i$, 
where $*$ denotes the quaternion product and $\tilde{\q}_i$ is the quaternion inverse of $\q_i$. 
The acceleration is given as $\ddx_i = \bmat{\ddp_i\tran & \dw_i\tran}\tran$.
In the following the deviation from a desired value $\bm a^d$ for any entity $\bm a \in \{\p_i,\dx_i,\ddx_i,\h_i\}$ is denoted as
$\Delta \bm a_i = \bm a_i-\bm a_i^d$. For a quaternion $\q_i$ the deviation from the desired orientation is denoted as $\Delta\q_i = \bmat{\Delta\eta_i, \Delta\fett{\epsilon}_i\tran}\tran = \q_i*\tilde{\q_i}^d$.

\subsection{System Dynamics}
The individual agents are modeled by the well known impedance equations as
\begin{equation}
	\M_i\Delta\ddx_i + \D_i\Delta\dx_i + \h_i^K\left(\x_i,\x_i^d \right) = \Delta\h_i \label{eq:impedance}
\end{equation}
where $\M_i$ and $\D_i$ denote the virtual mass and damping as
\begin{align}
	\M_i &= \bmat{m_i\I_3 & \0_3 \\ \0_3 & \J_i},	&& \D_i = \bmat{d_i\I_3 & \0_3 \\ \0_3 & \delta_i\I_3}
\end{align}
with design parameters $m_i,d_i,\delta_i \in \Rs_{>0}$ and inertia matrix $\bm{J}_i$. The endeffector wrench $\h_i = \bmat{\f_i\tran,\bm \tau_i\tran}\tran$ consists of the forces $\f_i\in\Rs^3$ and torques $\bm \tau_i\in\Rs^3$.
The vector~$\h_i^K$ represents the geometrically consistent stiffness~\cite{4639601} as
\begin{equation}
	\h_i^K\left( \x_i,\x_i^d\right) = \bmat{k_i\Delta\p_i \\ \kappa\strich_i \Delta\bm \epsilon_i}
\end{equation}
with~$\kappa\strich_i = 2\Delta\eta_i\kappa_i$ and scalar parameters $k_j,\kappa_j\in\mathbb{R}_{>0}$.
\begin{remark}
    For the sake of exposition the gains $m_i,d_i,\delta_i,k_i,\kappa_i$ are chosen to be scalar, but the results of this work can be easily adapted to matrix valued gains by applying the same techniques.
\end{remark}
The pose, velocity, and acceleration of the object $\x_o,\dx_o,\ddx_o$ are defined equivalent to the states of the agents $i$.
We assume the object frame to be located at the object \ac{com} and thus the object dynamics are given as
\begin{equation}
	\M_o\ddx_o+\C_o(\x_o,\dx_o) = \h_o \label{eq:object_dynamics}
\end{equation}
where $\h_o$ denotes the effective object wrench. The object mass/inertia matrix is given as $\M_o = blkdiag(m_o\I_3,\J_o)$ and the wrench containing the gravity- and Coriolis-effects is $\C_o = \bmat{ -m_o\bm g\\ \w_o\times\J_o\w_o}$,
where $m_o\in\Rs_{>0}$ and $\J_o\in\mathbb{R}^{3\times 3}$ denote the object mass and the positive definite symmetric inertia matrix, respectively, and $\bm g \in\mathbb{R}^{3}$ is the gravity vector.
Note that the inertia matrix $\J_o$ in the world frame is not constant if the object rotates, which can be problematic during the estimation process. To combat this we can express $\J_o$ as\looseness=-1
\begin{equation}
    \J_o = \R{o}{}^o\J_o\R{o}^T
\end{equation}
where ${}^o\J_o$ is the constant inertia matrix denoted in the object frame.

\subsection{Relative Kinematics}
The relative position and orientation of any two frames $i\neq j \in \{1,...,N,o\}$ are defined as
\begin{align}
	\pji &=\R{j}^T(\p_i- \p_j) \label{eq:pos_rel}\\ 
	\Rr{i}{j} &=\R{j}^T\R{i},\label{eq:ori_rel}
\end{align}
where $\R{j}$ is the rotation matrix corresponding to the quaternion $\q_j$.
Note that due to the rigidity assumption the relative position ${}^j\p_i$ and the orientation ${}^j\q_i$ are always constant.
In order to highlight this fact, and due to their importance in the following results, we denote the constant distances ${}^o\rn_{j,i} = \R{o}^T(\p_j-\p_i)$ and $\ori = \R{o}^T(\p_i-\p_o)$.
Solving~\eqref{eq:pos_rel} for $\p_j$ and differentiating, we obtain the relationship of translational velocities and accelerations between two frames as\looseness=-1
\begin{align}
	\dpn_j &= \dpn_i + S(\w_i)\pji \label{eq:vel_rel}\\
	\ddp_j &= \ddp_i + \underbrace{\left(S(\dw_i) + S(\w_i)^2\right)}_{\bm T(\w_i,\dw_i)} \pji. \label{eq:acc_rel}
\end{align}
In addition, since the relative orientation ${}^j\q_i$ is constant for all $i = 1,...,N$ it holds that 
$\w_o = \w_i = \w_j$, $\dw_o = \dw_i = \dw_j$
for all $j = 1,...,N$.

\subsection{Cooperative Manipulation}
The object wrench $\h_o$ can be related to the endeffector wrenches $\h_i$ as
\begin{align}
	\h_o = \underbrace{\bmat{\I_3 & \0_3 & \cdots & \I_3 & \0_3\\S(\fett{r}_1) & \I_3 & \cdots & S(\fett{r}_N) & \I_3 }}_{\eqcolon \G}\bar{\h} \label{eq:force_o}
\end{align} 
with the \textit{grasp matrix} $\G\in\mathbb{R}^{6\times 6N}$ and the vector of combined wrenches $\bar{\fett{h}}=-[\fett{h}_1\tran,\cdots,\fett{h}_N\tran]\tran$.
Similarly
\begin{equation}
	\dx = \G\tran\dx_o, \label{eq:vel_trans}
\end{equation}
holds, where $\dx$ denotes the concatenated vector of the individual agent velocities $\dx_i$.
For the cooperative manipulation task, some coordination strategy is required in order to avoid internal stress in the object. A common solution~\cite{Erhart.2016} is obtained choosing the desired velocities according to~\eqref{eq:vel_trans} and compensating the object dynamics by setting $\h^d = -\G^+\h_o^d$, with a desired object wrench $\h_o^d$ compensating the object dynamics, and generalized inverse of the grasp matrix $\G^+$.
However, the grasp matrix $\G$ and the compensation terms $\h_o^d$ depend on the typically unknown parameters $\ori,m_o,\J_o$ and an estimation strategy is indispensable for a successful coordination of the individual agents. While such estimation strategies exist, those are typically centralized and do not provide a measure of uncertainty.

\subsection{Local Information and Communication}
We assume that at each time instance $k$ each agent $i$ can measure its own state $\z_i^{(k)} \!=\! \bmat{\x_i\tran & \dx_i\tran & \ddx_i\tran}\tran$, which we will term the local information of agent $i$. 
We pose the following additional assumption on locally available information.
\begin{assumption}
\label{ass:local_info}
The reference signals for all quantities of all agents, as well as all parameters of the dynamics~\eqref{eq:impedance} are known to each agent.
\end{assumption}
The availability of the desired quantities during the estimation process might seem like a strong assumption. However, typically an identification trajectory is designed before task execution and can be locally stored at each agent. While such an approach might result in internal stress on the object, this can be avoided by updating the identification strategy online using the consensus estimates presented in this work. The second part of~\aref{ass:local_info} is not restrictive, since the constant parameters can be either exchanged before task execution or propagated through the network with methods similar to the ones presented in this work.
\begin{assumption}
The relative rotation matrices ${}^i\R{o}$ are known by each agent.
\end{assumption}
This assumption consists of two parts. First, note that ${}^i\R{o} = \R{i}^T\R{o} = \text{const.}$ and as a result it is required that each agent has information about the initial object orientation $\R{o}$. However, since the object frame can be oriented arbitrary this essentially means that the agents merely have to decide on a common initial orientation, which can be achieved via consensus algorithms on $SE(3)$ or methods similar to the ones presented in~\cite{Tron2009}. Second, we assume that each agent $i$ knows the relative orientation ${}^j\R{o}$ for all agents $j$. By recalling that ${}^i\R{o}$ is constant, the relative orientations can be shared before task execution over a communication network.

Finally, we allow the agents to communicate on a communication graph $\Gs \!=\! \{\Vs,\Es\}$, where $\Vs\!\subseteq\!\{1,...,N\}$ is the vertex set, representing the individual agents and $\Es\!\subseteq\!\Vs\!\times\!\Vs$ is the edge set, where $(i,j)\!\in\!\Es$ if agent $i$ and $j$ can communicate with each other. 
The edges of the graph can be compactly represented through the weighted adjacency matrix~$\Adj\!\in\!\mathbb{R}^{N\times N}$, where~$A^{ij}\!>\!0$ if~$(j,i)\!\in\!\Es$. 
The communication network must satisfy the following properties~\cite{Zhu2010}.
\begin{assumption}
\label{as:connectivity}
    The communication graph $\Gs$ is strongly connected and balanced, i.e., $\Adj\bm{1}\!=\!\Adj\tran\bm{1}\!=\!\bm{1}$.
    Moreover, there exists a positive constant $\alpha$ such that i) $A^{ii}\!\geq\! \alpha$ for all~$i$, ii) $A^{ij}\!\in\!\left\lbrace 0 \right\rbrace\cup\left[\alpha,1\right]$, for all $i,j$, iii) $\sum_{j=1}^{n}\!A^{ij}\!=\!1$, for all $i$.\looseness=-1
\end{assumption}

The goal of this work is to first estimate the parameters $\ori,m_o,\J_o$ by using only locally available information. In a second step we want to increase the accuracy of the estimation by allowing the agents to exchange their estimates along the edges of the communication graph.

\section{DISTRIBUTED BAYESIAN ONLINE LEARNING}
\label{sec:main}

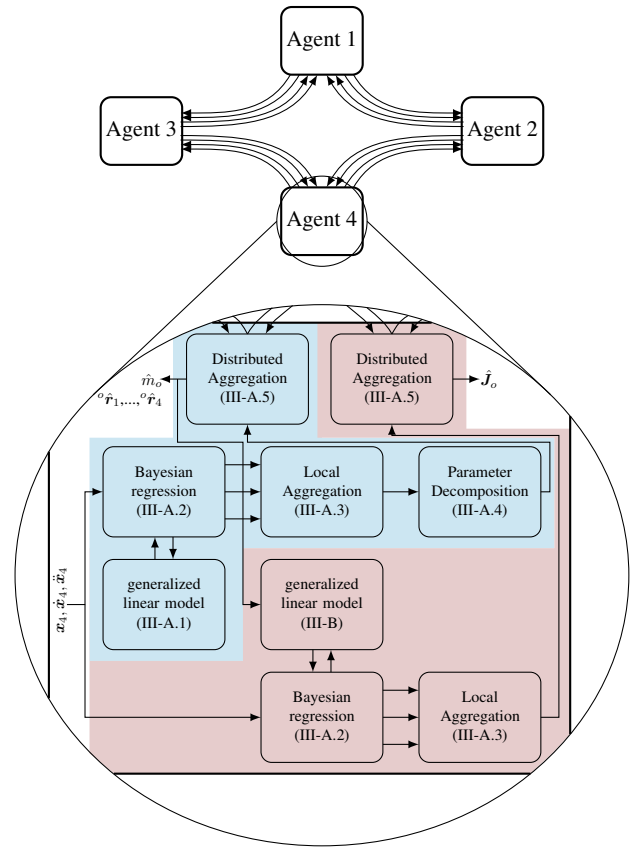
\begin{figure}
\vspace{0.2cm}
    \centering
    \begin{tikzpicture}[scale=0.6, every node/.style={scale=0.6}]
    
    \node[align=center,draw=black, rectangle,minimum height=1.5cm, minimum width=1.725cm, rounded corners, thick] at (7cm,10.cm) {\Large Agent 1};
    \node[align=center,draw=black, rectangle,minimum height=1.5cm, minimum width=1.725cm, rounded corners, thick] at (11cm,8.cm) {\Large Agent 2};
    \node[align=center,draw=black, rectangle,minimum height=1.5cm, minimum width=1.725cm, rounded corners, thick] at (3cm,8.cm) {\Large Agent 3};
    \node[align=center,draw=black, rectangle,minimum height=1.5cm, minimum width=1.725cm, rounded corners, thick] at (7cm,6.cm) {\Large Agent 4};
    
    \draw[->, >=latex] (3.867,7.9) to[out=-0,in=120] (6.9,6.75);
    \draw[->, >=latex] (3.867,7.8) to[out=-0,in=120] (6.7,6.75);
    \draw[->, >=latex] (6.5,6.75) to[out=120,in=0] (3.867,7.7);
    \draw[->, >=latex] (6.3,6.75) to[out=120,in=0] (3.867,7.6);
    
    \draw[->, >=latex] (10.1375,7.9) to[out=-180,in=60] (7.1,6.75);
    \draw[->, >=latex] (10.1375,7.8) to[out=-180,in=60] (7.3,6.75);
    \draw[->, >=latex] (7.5,6.75) to[out=60,in=-180] (10.1375,7.7);
    \draw[->, >=latex] (7.7,6.75) to[out=60,in=-180] (10.1375,7.6);
    
    \draw[->, >=latex] (3.867,8.1) to[out=-0,in=-120] (6.9,9.25);
    \draw[->, >=latex] (3.867,8.2) to[out=-0,in=-120] (6.7,9.25);
    \draw[->, >=latex] (6.5,9.25) to[out=-120,in=0] (3.867,8.3);
    \draw[->, >=latex] (6.3,9.25) to[out=-120,in=0] (3.867,8.4);
    
    \draw[->, >=latex] (10.1375,8.1) to[out=-180,in=-60] (7.1,9.25);
    \draw[->, >=latex] (10.1375,8.2) to[out=-180,in=-60] (7.3,9.25);
    \draw[->, >=latex] (7.5,9.25) to[out=-60,in=-180] (10.1375,8.3);
    \draw[->, >=latex] (7.7,9.25) to[out=-60,in=-180] (10.1375,8.4);
    
    \draw (7,6.) circle (1.0);
    \draw (6,6) -- (1.93,2.15);
    \draw (8,6) -- (12.07,2.15);
    
    \fill[blue!20] (1.85cm,1.2cm) -- (3.7cm,1.2cm) -- (3.7,3.4) -- (4.5cm,3.75cm) -- (7cm,3.75cm) -- (7cm,1.2cm) -- (12.15cm,1.2cm) -- (12.15cm,-1.25cm) -- (5.25cm,-1.25cm) --(5.25cm,-3.75cm)-- (1.85cm,-3.75cm);
    \fill[red!20] (2.4cm,-6.25cm) -- (11.62cm, -6.25cm) -- (12.1,-5.8) -- (12.5,-5.37) -- (12.5,1.4) -- (10.2,1.4) -- (10.2,3.45) -- (9.42,3.75) -- (6.9,3.75) -- (6.9,1.2) -- (12.15,1.2) -- (12.15,-1.25) -- (8.5cm,-1.25cm) -- (5.25cm, -1.25cm) -- (5.25cm,-3.75cm) -- (1.85cm,-3.75cm) -- (1.85cm,-1.25cm) -- (1.85,-5.8);
    \draw[thick] (4.57,3.75) -- (9.43,3.75);
    \draw[thick] (2.39,-6.25) -- (11.62,-6.25);
    \draw[thick] (0.95,0.89) -- (0.95,-4.58);
    \draw[thick] (12.5,1.68) -- (12.5,-5.37);
    \draw[] (7, -1.85) ellipse (6.8 and 6);

    \node[align=center,draw=black, rectangle,rounded corners, minimum height=2cm, minimum width=2.7cm] at (3.5cm,-2.5cm) {generalized\\ linear model\\ (\ref{sssec:trans_model})};
    \node[align=center,draw=black, rectangle,rounded corners, minimum height=2cm, minimum width=2.7cm] at (3.5cm,0cm) {Bayesian\\ regression\\ (\ref{sssec:trans_blr})};
    \node[align=center,draw=black, rectangle,rounded corners, minimum height=2cm, minimum width=2.7cm] at (7cm,0cm) {Local\\ Aggregation\\ (\ref{sssec:trans_gpoe})};
    \node[align=center,draw=black, rectangle,rounded corners, minimum height=2cm, minimum width=2.7cm] at (10.5cm,0cm) {Parameter\\ Decomposition\\ (\ref{sssec:trans_ratio})};
    \node[align=center,draw=black, rectangle,rounded corners, minimum height=2cm, minimum width=2.7cm] at (5.35cm,2.5cm) {Distributed\\ Aggregation\\ (\ref{sssec:trans_dac})};
    
    \node[align=center,draw=black, rectangle,rounded corners, minimum height=2cm, minimum width=2.7cm] at (7cm,-2.5cm) {generalized\\ linear model\\ (\ref{sssec:rot_model})};
    \node[align=center,draw=black, rectangle,rounded corners, minimum height=2cm, minimum width=2.7cm] at (7cm,-5cm) {Bayesian\\ regression\\ (\ref{sssec:trans_blr})};
    \node[align=center,draw=black, rectangle,rounded corners, minimum height=2cm, minimum width=2.7cm] at (10.5cm,-5cm) {Local\\ Aggregation\\ (\ref{sssec:trans_gpoe})};
    \node[align=center,draw=black, rectangle,rounded corners, minimum height=2cm, minimum width=2.7cm] at (8.55cm,2.5cm) {Distributed\\ Aggregation\\ (\ref{sssec:trans_dac})};
    
    \node[rotate=90] at (1.2,-2.4) {$\bm{x}_4,\dot{\bm{x}}_4,\ddot{\bm{x}}_4$};
    \node[align=right] at (2.75,2.25) {$\hat{m}_o$\\ ${}^o\hat{\bm{r}}_1$,\mydots,${}^o\hat{\bm{r}}_4$};
    \node at (10.7,2.5) {$\hat{\bm{J}}_o$};

    \draw[->, >=latex] (1.35cm,-2.5cm) -- (1.75cm,-2.5cm) -- (1.75cm,0cm) -- (2.15cm,0cm);
    \draw[->, >=latex] (1.75cm,-2.5cm) -- (1.75cm,-5cm) -- (5.65cm,-5cm);
    \draw[->, >=latex] (3.3cm,-1.5cm) -> (3.3cm,-1cm);
    \draw[->, >=latex] (3.7cm,-1cm) -> (3.7cm,-1.5cm);
    \draw[->, >=latex] (4.85cm,0cm) -> (5.65cm,0cm);
    \draw[->, >=latex] (4.85cm,0.6cm) -> (5.65cm,0.6cm);
    \draw[->, >=latex] (4.85cm,-0.6cm) -> (5.65cm,-0.6cm);
    \draw[->, >=latex] (8.35cm,0cm) -> (9.15cm,0cm);
    \draw[->, >=latex] (11.85cm,0cm) -- (12.05cm,0cm) -- (12.05,1.1) -- (5.35,1.1) -- (5.35,1.5);
    \draw[->, >=latex] (4cm,2.5cm) -> (3.4cm,2.5cm);
    \draw[->, >=latex] (3.8,2.5) -- (3.8,1.1) -- (5.25,1.1) -- (5.25,-2.5) -- (5.65,-2.5);

    \draw[->, >=latex] (8.35cm,-5cm) -> (9.15cm,-5cm);
    \draw[->, >=latex] (8.35cm,-5.6cm) -> (9.15cm,-5.6cm);
    \draw[->, >=latex] (8.35cm,-4.4cm) -> (9.15cm,-4.4cm);
    \draw[->, >=latex] (11.85cm,-5cm) -- (12.25cm,-5cm) -- (12.25,1.25) -- (8.55,1.25) -- (8.55,1.5);
    \draw[->, >=latex] (9.9cm,2.5cm) -> (10.5cm,2.5cm);
    
    \draw[->, >=latex] (6.8cm,-3.5cm) -> (6.8cm,-4cm);
    \draw[->, >=latex] (7.2cm,-4cm) -> (7.2cm,-3.5cm);
    
    \draw[] (8.55,3.5) to[out=60,in=-150] (8.9,3.9);
    \draw[] (8.55,3.5) to[out=120,in=-40] (8.0,4.08);
    \draw[->, >=latex] (9.22,3.82) to[out=-150,in=60] (8.95,3.5);
    \draw[->, >=latex] (7.55,4.13) to[out=-40,in=120] (8.15,3.5);
    
    \draw[] (5.35,3.5) to[out=60,in=-140] (5.9,4.08);
    \draw[] (5.35,3.5) to[out=120,in=-30] (5.02,3.9);
    \draw[->, >=latex] (6.35,4.13) to[out=-140,in=60] (5.75,3.5);
    \draw[->, >=latex] (4.7,3.8) to[out=-30,in=120] (4.95,3.5);
    
    \end{tikzpicture}
    \caption{Overview of the proposed distributed Bayesian learning framework. The rotational regression problem (red) depends on the estimates of the translational regression problem (blue), leading to a sequential estimation structure. For both problems we provide linear models and use Bayesian regression to obtain estimates of unknown parameters, followed by a local aggregation and parameter decomposition for the the translational estimator. Finally, the estimates from all agents are aggregated in a distributed fashion.}
    \label{fig:overview}
\end{figure}

In this section we will derive a generalized model for cooperative manipulation which is linear in the parameters $m_o\ori,m_o,\J_o$ and depends only on information locally available at each agent. Based on this model and by applying Bayesian principles, we proceed by presenting our novel distributed estimator for the required parameters.
An overview of the general framework is presented in \fref{fig:overview}. The manipulation task is split into two parts corresponding to the translational and rotational degrees of freedom, each with a separate estimator. This results in a sequential estimation procedure, where the results from the translational estimator are used in the model for the rotational one. 

\subsection{Translational Regression Problem}
\subsubsection{Generalized Linear Model}\label{sssec:trans_model}
To obtain a distributed algorithm for the translational regression problem, we start by transforming all relevant equations such that only locally available information is used.
Substituting the equations~\eqref{eq:pos_rel},~\eqref{eq:vel_rel},~\eqref{eq:acc_rel} in the original impedance model~\eqref{eq:impedance} we can use the states of agent $i$ to express the force $\f_j$ as
\begin{align}
    \f_j &\!=\! m_j\left(\ddp_i \!+\! \Scal \Drij{j,i} \!-\! \ddp_j^d  \right)
	 \!+\! d_j \left(\dpn_i \!+\! S(\w_i)\Drij{j,i} \!-\! \dpn_j^d\right)\nonumber\\
	 &\!+\! k_j\left(\p_i \!+\! \Drij{j,i} \!-\! \p_j^d\right) \!+\! \f_j^d.
\end{align}
Combining this with~\eqref{eq:force_o} gives us an expression for the effective forces acting on the object. Finally, by transforming the object states using the equations~\eqref{eq:pos_rel},~\eqref{eq:vel_rel},~\eqref{eq:acc_rel}, the left-hand side of the object dynamics~\eqref{eq:object_dynamics} can also be expressed with locally available information of agent $i$. Separating the resulting equations into known and unknown entities yields a model of the form
\begin{equation}
	\bm y_i = \bm \phi_i\tran\thet_i \label{eq:model}
\end{equation}\textbf{}
for each agent $i$, which is linear in the parameters

\begin{equation}
\begin{split}
    \thet_i &= \bmat{\thet_{i,1}\tran & \thet_{i,2}\tran}\tran\\    
    \thet_{i,1} &= \bmat{{}^o\rn_{1,i}\tran & \cdots & {}^o\rn_{i-1,i}\tran & {}^o\rn_{i+1,i}\tran & \cdots & {}^o\rn_{N,i}\tran}\tran\\
    \thet_{i,2} &=  \bmat{m_o\ori\tran & m_o}\tran
\end{split}    
\end{equation}
and requires only local information, where
\begin{equation}
	\begin{split}
	\bm y_i(\z_i) &= \sumaj\left[ m_j\ddp_j^d+d_j\dpn_j^d+k_j\p_j^d -\f_j^d \right]\\
	&  -m_{c}\ddp_i - d_{c}\dpn_i - k_{c}\p_i\\
	\bm\phi_i(\z_i)\tran &= \bmat{\bm \phi_{1,i}(\z_i) & \bm \phi_{2,i}(\z_i) & \bm \phi_{3,i}(\z_i)}\\
	\bm \phi_{1,i}(\z_i) &= \Scal \M_{w,i} +S(\fett{\omega}_i) \D_{w,i} + \K_{w,i}\\
	\bm \phi_{2,i}(\z_i) &= -\Scal\R{0}\\
	\bm \phi_{3,i}(\z_i) &= \ddp_i-\bm g.\\
	\end{split}\label{eq:model_trans_long}
\end{equation}
and for any $\bm B_{w,i} \in \{\M_{w,i},\D_{w,i},\K_{w,i}\}$ and respective parameters $b_i\in \{m_i,d_i,k_i\}$ we define $\bm B_{w,i} = \bmat{b_1\R{o}, ... ,b_{i-1}\R{o},b_{i+1}\R{o}, ... b_N\R{o}}$ and $b_{c} = \sum_{i=1}^{N}b_i$.
Note that in~\eqref{eq:model} the left side is an entirely determined $3$-element vector of forces. The right side consists of a~\mbox{$3\times (3N+1)$} composed matrix of measurable inputs and a~$(3N+1)$ vector of unknown parameters whose values we later wish to determine.

\subsubsection{Bayesian Linear Regression}\label{sssec:trans_blr}
\label{sec: BLR}
We use Bayesian linear regression to find the estimates $\hat{\thet}_i\in\Rs^{3N+1}$ of the unknown parameters $\thet_i$. In contrast to standard linear regression, where the estimates are obtained by minimizing the quadratic error within the linear model~\cite{Fahrmeir.2009}, this allows us to assess the uncertainty of the estimates and incorporate prior knowledge about the parameter values or possible noise perturbations.
We place a prior probability 
distribution~$p_i(\thet_i):  \Rs^{3N+1}\rightarrow\Rs_{\geq 0}$ upon the parameters, which is given by
\begin{equation}
	p_i(\thet_i) = \mathcal{N}(\thet_i|\bm \mu_{i}^{(0)},\Sigma_{i}^{(0)})
\end{equation}
with the initial mean~$\bm \mu_{i}^{(0)}\in\Rs^{3N+1}$ and covariance matrix~$\Sigma_{i}^{(0)}\in\Rs^{(3N+1)\times (3N+1)}$ \cite{Bishop2006}.
In order to derive closed-form expressions for inference, we assume for the moment that training targets~$t_{i,m} = y_{i,m}+\varepsilon$ are output values~$y_{i,m}$ perturbed by homoscedastic normally distributed i.i.d. noise~$\varepsilon \sim \mathcal{N}(0,\beta^{-1})$, where the index $m=1,2,3$ denotes the $m$th element of the three dimensional vectors $\bm y_i$ and $\bm t_i$.
Given a current estimate $\bm \mu_{i,m}^{(k)}\in\Rs^{3N+1}$, $\Sigma_{i,m}^{(k)}\in\Rs^{(3N+1)\times (3N+1)}$ and new data $\z_i^{(k+1)}$, the estimate can be updated by considering $p(\thet_i| \bm{t}_{i,m}^{(k)})$ as prior, i.e.,
\begin{align}
	p(\thet_i|\bm{t}_{i,m}^{(k+1)}) = \frac{p(t_{i,m}^{(k+1)}|\thet_i,\bm{t}_{i,m}^{(k)})p(\thet_i|\bm{t}_{i,m}^{(k)})}{p(t_{i,m}^{(k+1)}|\bm{t}_{i,m}^{(k)})},
\end{align}
where $\bm{t}_{i,m}^{(k)}$ denotes the stacked vector of output values $t_{i,m}$ at~$k$ different input instances. 
Since both the noise $\varepsilon$ and the parameter vector $\thet_i$ are Gaussian random variables, it follows from basic properties of Gaussian distributions that the posterior $p(\thet_i|t_{i,m}^{(k+1)})$ at each time step $k+1$ is also Gaussian with mean and covariance matrix
\begin{align}
\label{eq:mean}
	\!\mulm{k+1} &\!=\! \siglm{k+1}\!\left( (\siglm{k})\inv\!\mulm{k}\!+\!\beta\bm{\phi}_ (\z^{(k+1)})t_{i,m}^{(k+1)}\right)\! \\
 	\!\siglm{k+1} &\!=\! \left((\siglm{k})\inv+\beta\bm\phi_i(\z^{(k+1)})\bm{\phi}_i\tran(\z^{(k+1)})\right)\inv.
	\label{eq:covar}
\end{align}
This gives an iterative update rule and allows for learning with constant update complexity.

\subsubsection{Generalized Product of Expert Aggregation}\label{sssec:trans_gpoe}
Until now we have three estimators for each agent, which each have their individual estimate $p(\thet_i|\bm{t}_{i,m}^{(k+1)})$. In order to combine these estimates without losing information about the estimates' variance or letting estimates with high uncertainty corrupt the aggregated mean, we choose a product of experts approach \cite{Cao2014}.
The key idea is to model a target probability distribution as the product of multiple densities, each provided by an expert. Experts are in this case linear Bayesian regression estimators, such that the product distribution
\begin{align}
    p(\thet_i|\bm t_{i,m}^{(k)})\propto\prod_{m=1}^3 p(\thet_i|\bm t_{i,m}^{(k)})
\end{align}
is Gaussian with mean and covariance matrix
\begin{align}
    \mul{k}&=\sigl{k}\left(\sum_{m=1}^3 (\siglm{k})\inv \mulm{k}\right)\\
    \sigl{k}&=\left( \sum_{m=1}^3(\siglm{k})\inv \right)\inv.
\end{align}
However, this can be problematic since distributions without data have lower variance than the individual distributions and as a result the contribution of experts has to be normalized, which leads to the generalized product of experts as
\begin{align}
\label{eq:gpoe_mean}
    \mul{k}&=\frac{1}{3}\sigl{k}\left(\sum_{m=1}^3 (\siglm{k})\inv \mulm{k}\right)\\
    \sigl{k}&=3\left( \sum_{m=1}^3(\siglm{k})\inv \right)\inv.
    \label{eq:gpoe_var}
\end{align}
Due to the structure of this aggregation, it exhibits the beneficial property that a low uncertainty
in a single estimator, i.e., a covariance matrix $\Sigma_{i,m}^{(k)}$ with small entries, 
is sufficient to achieve an overall low uncertainty. This effect can also
be observed in the mean $\bm{\mu}_i^{(k)}$, to which individual estimates $\mulm{k}$ 
with high corresponding variance have a low impact.

\subsubsection{Decomposition of Regression Parameters}\label{sssec:trans_ratio}
For the sake of readability we denote the elements of $\bm \mu_i^{(k)}$ corresponding to a certain parameter $\bm a \in \bm \theta$ as $\estim{\bm a}$ and accordingly $\cov{\bm a}$ denote the corresponding diagonal elements of the covariance matrix $\bm \Sigma_i^{(k)}$. If not stated otherwise all quantities $\estim{\bm a}$ and $\cov{\bm a}$ refer to agent $i$ at time-step $k$.
By recalling \sref{sec:problem} we want to find an estimate for the parameters $\orj,m_o,\bm J_o$, $j=1,\ldots,N$, while with the presented results, each agent $i$ obtains estimates for $m_o, m_o\ori, \orj-\ori$. The \ac{com} $\orj$, is not a direct output
of the linear Bayesian regression. When dealing with deterministic estimates, one could straightforwardly
obtain $\ori$ by dividing the estimates for $m_o\ori$ and $m_o$. For Gaussian random variables, 
the stochastic equivalent to the deterministic division is the ratio distribution \cite{Diaz-Frances2013}. 
The exact distribution of the ratio of two Gaussian random variables can be calculated in closed-form
under the assumption of strictly positive (or negative) mean values, but the resulting expression 
is rather complicated, which is prohibitive for further derivations. Therefore, we approximate 
the ratio distribution by a Gaussian distribution, whose mean and variance follow from a Taylor
approximation of the exact ratio distribution. 
Considering only the marginal distributions described by the means $\estim{m_o\ori}$ and the vector of diagonal elements $\cov{m_o\ori}$ of the covariance matrix $\Sigma_i^{(k)}$, this yields the following identities 
\begin{align}
\label{eq:ratiomean}
    \estim{\ori} &= \frac{\estim{m_o\ori}}{\estim{m}_o}\\
    \cov{\ori} &= (\estim{\ori})^2\left(\frac{\cov{m_o\ori}}{\estim{m_o\ori}^2}+\frac{\cov{m}_o}{\estim{m}_o^2} \right).
    \label{eq:ratiovar}
\end{align}
The error between the approximate Gaussian distribution with mean \eqref{eq:ratiomean} and standard deviation \eqref{eq:ratiovar} and the exact ratio distribution can be shown to be bounded if $\frac{\cov{m_o\ori}}{\estim{m_o\ori}^2}$ and $\frac{\cov{m}_o}{\estim{m}_o^2}$ are sufficiently
small \cite{Diaz-Frances2013}. Since the variance of the linear Bayesian estimator decreases with a growing number
of suitable data, the Gaussian distributions eventually satisfy this condition, such that modelling
the exact ratio distribution using a Gaussian distribution is justified.

This estimate of $\ori$ directly allows us to obtain estimates for $\orj$, $j\neq i$. Since all
considered distributions are Gaussian, $\orj$ can be estimated by adding two Gaussian random variables, 
which yields a posterior Gaussian distribution with mean and variance
\begin{align}
\label{eq:divmean}
    \estim{\orj}&=\estim{\ori}+\estim{{}^o\rn_{j,i}}\\
    \cov{\orj}&=\cov{\ori}+\cov{{}^o\rn_{j,i}}.
    \label{eq:divvar}
\end{align}
Finally, we can define the vector of means and diagonal covariance elements of the local kinematic and dynamic parameters for each agent $i$ as
\begin{align}
    \label{eq:loc_dist_mu}
    \hat{\bm\mu}_i^{(k)} &= \bmat{\estim{{}^o\rn_{1}}\tran, \dots, \estim{{}^o\rn_{N}\tran},\estim{m}_o}\tran\\
    \hat{\bm\sigma}_i^{(k)} &= \bmat{\cov{{}^o\rn_{1}}\tran, \dots, \cov{{}^o\rn_{N}\tran},\cov{m}_o}\tran.
    \label{eq:loc_dist_sig}
\end{align}

\subsubsection{Distributed Model Aggregation}\label{sssec:trans_dac}
It should be noted that the estimation up until this point is done by each agent individually, which results in $N$ estimates $\hat{\bm{\mu}}_i^{(k)}$, $\hat{\bm{\sigma}}_i^{(k)}$. In order to combine the different 
predictive distributions, we can follow a generalized product of experts approach in principle again. 
However, the communication between agents is restricted according to a communication graph $\mathcal{G}$ in
the considered multi-agent setting, such that the direct aggregation of individual predictions is not 
possible. In order to mitigate this issue, we formulate the generalized product of experts aggregation
as dynamic average computation, such that consensus algorithms are applicable. 

In order to achieve this, we transform the local distribution parameters $\hat{\bm{\mu}}_i^{(k)},$ $\hat{\bm{\sigma}}_i^{(k)}$ defined in \eqref{eq:loc_dist_mu}-\eqref{eq:loc_dist_sig} using
\begin{align}
\label{eq:psi}
    \bm{\psi}_i\!=\!\begin{bmatrix}
    \left(\frac{\hat{\bm\mu}_i^{(k)}}{\hat{\bm\sigma}_i^{(k)}}\right)\tran&
    \left(\frac{1}{\hat{\bm\sigma}_i^{(k)}}\right)\tran
    \end{bmatrix}\tran\!.
\end{align}
To improve to the overall estimate we take the average value of the individual estimates, computed in distributed fashion using a consensus type algorithm.
Hence, we define the consensus states
$\bm{\xi}_i\in\mathbb{R}^{6N+2}$ with the dynamics
\begin{align}
    \bm{\xi}_i^{(k+1)}&=\bm{\xi}_i^{(k)} \!+ \sum_{j\neq i}^N A^{ij}(\bm{\xi}_j^{(k)}\!-\bm{\xi}_i^{(k)})\!+\bm{\psi}_i^{(k)}\!-\bm{\psi}_i^{(k-1)}
\end{align}
and initial state $\bm{\xi}_i^{(0)}=\bm{\psi}_i^{(0)}$, following the approach proposed in \cite{Zhu2010}. This dynamical system has been shown 
to exhibit a bounded consensus error under Assumption~\ref{as:connectivity}. 
Therefore, the local consensus states $\bm{\xi}_i^{(k)}$ can directly be used for computing the 
mean and variance of the predictive Gaussian distributions in a distributed way through 
\begin{align}
    \begin{bmatrix}
    \tilde{\bm{\mu}}_i^{(k)}\\
    \tilde{\bm{\sigma}}_i^{(k)}
    \end{bmatrix}&=\bm{\zeta}\left( \bm{\xi}_i^{(k)} \right),
\end{align} where $\bm{\zeta}:\mathbb{R}^{6N+2}\rightarrow\mathbb{R}^{6N+2}$ is defined as
\begin{align}
    \bm{\zeta}(\bm{\psi})&\!=\!\begin{bmatrix}
    \frac{\psi_1}{\psi_{3N+2}}\!&\!\cdots\!&\!\frac{\psi_{3N+1}}{\psi_{6N+2}}\!&\!\frac{1}{\psi_{3N+2}}\!&\!\cdots\!&\!\frac{1}{\psi_{6N+2}}
    \end{bmatrix}\tran\!,\!
\end{align}
where $\psi_m$ denotes the $m$th entry of the vector $\bm \psi$.
Hence, a simple dynamic average consensus algorithm in combination with a generalized product of 
expert aggregation allows an efficient combination of the predictive distributions resulting from local
estimators.

\subsubsection{Learning Error Bound}
Due to the strong theoretical foundation of the employed methods, it is straightforward to derive a probabilistic error bound for the estimated parameters. In order to bound the estimation of Bayesian linear regression, we make the following assumption.
\begin{assumption}
\label{ass:noise}
We assume homoscedastic normally distributed errors~$\varepsilon \sim \mathcal{N}(0,\beta^{-1})$ of the output~$\bm y_i$.
\end{assumption}
Although this assumption does not reflect more realistic scenarios, where all observed states $\bm{z}$ are noisy, it is only used to streamline the proof of the following theorem. In fact, it is straightforward to derive error bounds for the estimates obtained from Bayesian linear regression under more general noise distributions since the function $\bm{\phi}(\cdot)$ defines a kernel. This allows the application of error bounds from Gaussian process regression, which admit, e.g., sub-Gaussian noise~\cite{Chowdhury2017a} and arbitrary bounded noise~\cite{Maddalena2020}. As the necessary derivations are rather cumbersome, they are omitted here due to space limitations.
For bounding the error caused by the ratio distribution, we require the following assumption.\looseness=-1
\begin{assumption}\label{ass:eta}
For $\delta\in(0,1)$, define $\bm{\eta}_i^{(k)}\in\mathbb{R}^{3N+1}$ as 
\begin{align}
\label{eq:eta}
    \bm{\eta}_i^{(k)}&\!=\! \sqrt{2\log\!\left(\!\frac{6N(3N\!+\!1)}{\delta}\!\right)}\sum_{m=1}^3\left|\frac{1}{3}\Sigma_i^{(k)}\left(\Sigma_{i,m}^{(k)}\right)\inv \right|\bm{\sigma}_{i,m}^{(k)},
\end{align}
where $\Sigma_{i,m}^{(k)}$, $\Sigma_i^{(k)}$ are introduced in \eqref{eq:covar}, \eqref{eq:gpoe_var}, respectively, 
and $\bm{\sigma}_{i,m}^{(k)}\!\!=\!\mathrm{diag}(\Sigma_{i,m}^{(k)})$.
Let $\eta_{i,m_o}^{(k)}$ and $\left(\estim{m}_o\right)^{(k)}_i$ 
be the elements of $\bm{\eta}_i^{(k)}$ and $\bm{\mu}_i^{(k)}$ corresponding to the parameter $m_o$, respectively. 
Then, there exists a $K\in\mathbb{N}$ such that
\begin{align}
\label{eq:divCond}
    \left|\left(\estim{m}_o\right)^{(k)}_i\right|-\bm{\eta}_{i,m_o}^{(k)} >0\quad \forall i=1,\ldots,N,~\forall k\geq K.
\end{align}
\end{assumption}
This assumption is not restrictive since $\bm{\eta}_i^{(k)}$ is non-increasing with respect to $k$, 
and a decrease can be guaranteed through a sufficient excitation of the system. Therefore, this 
condition basically requires that the trajectory $\x^d$ is chosen suitably.
Finally, we employ the following assumption to ensure a vanishing error of the dynamic average consensus.
\begin{assumption}\label{ass:psibound}
It holds that 
$\lim_{k\rightarrow\infty}1/\hat{\bm\sigma}_{i}^{(k)}>\bm 0$.
\end{assumption}
Since we do not assume an optimal excitation signal, we cannot expect asymptotically vanishing covariance matrices
$\siglm{k}$ in general. Therefore, $\lim_{k\rightarrow\infty}1/\hat{\bm\sigma}_{i}^{(k)}>\bm 0$ is often
satisfied in practice. Even if \cref{ass:psibound} does not hold, the static consensus error can be bounded and is usually very small, such that \cref{th:thetabound} holds approximately.
Based on these assumptions, we can bound the learning error as shown in the following theorem.
\begin{theorem}
\label{th:thetabound}
Consider a communication graph satisfying Assumption~\ref{as:connectivity}, 
and training data satisfying Assumptions~\ref{ass:noise}-\ref{ass:psibound}. 
Then, the error between estimated parameters $\tilde{\bm{\mu}}_i^{(k)}$ and unknown parameters $\bm{\vartheta}=[{}^o\bm{r}_1\tran\ \cdots \ {}^o\bm{r}_N\tran\ m_o]\tran$ satisfies \looseness=-1
\begin{align}
    \lim_{k\rightarrow\infty} P\left(|\tilde{\bm{\mu}}_i^{(k)}\!-\!\bm{\vartheta}|\leq\hat{\bm{\eta}}^{(k)},~\forall i\!=\!1,\ldots,N\right)\geq 1\!-\!\delta
    \label{eq:thetabound}
\end{align}
for $\hat{\bm{\eta}}^{(k)}\!=\!\left[\!\left(\hat{\bm{\eta}}_{{}^o\bm{r}_1}^{(k)}\right)\tran\!\ \cdots\ \left(\hat{\bm{\eta}}_{{}^o\bm{r}_N}^{(k)}\right)\tran\!\ \hat{\eta}_{m_o}^{(k)}\right]\tran$ composed of\allowdisplaybreaks
\begin{align}\label{eq:eta_i}
    \hat{\bm{\eta}}_{{}^o\bm{r}_j}^{(k)}&=\sum_{i=1}^N \frac{ \left( \bm{\eta}_{{}^o\bm{r}_{i}}^{(k)}+\bm{\eta}_{i,{}^o\rn_{j,i}}^{(k)} \right) }{ \left(\cov{\orj}\right)^{(k)}_i\sum_{n=1}^N \frac{1}{\left(\cov{\orj}\right)^{(k)}_n}}
\end{align}
\vspace{-0.4cm}
\begin{align}
    \hat{\eta}_{m_o}^{(k)}&=\sum_{i=1}^N\frac{\eta_{i,m_o}^{(k)}}{\left(\tilde{m}_{o}\right)^{(k)}_i\sum_{n=1}^N \frac{1}{\left(\tilde{m}_{o}\right)^{(k)}_n}} \label{eq:eta_mo}\\
        \bm{\eta}_{{}^o\bm{r}_{i}}^{(k)}&=\max\!\left\{\!\left|\frac{\left(\estim{m_o\ori}\!\right)^{\!\!(k)}_i}{\left(\estim{m}_o\right)^{(k)}_i}\!-\!\frac{\left(\!\estim{m_o\ori}\!\right)^{\!\!(k)}_i\!\!\pm\bm{\eta}_{i,m_o{}^o\bm{r}_i}^{(k)}}{\left(\estim{m}_o\right)^{(k)}_i\!\!\pm\eta_{i,m_o}^{(k)}}\right|\!\right\}\!,
        \label{eq:rho}
\end{align}
where $\bm{\eta}_{i,{}^o\rn_{j,i}}^{(k)}$, $\bm{\eta}_{i,m_o{}^o\bm{r}_i}^{(k)}$ and $\eta_{i,m_o}^{(k)}$ are
the elements of $\bm{\eta}_i^{(k)}$ corresponding to the parameters ${}^o\rn_{j,i}$, $m_o{}^o\bm{r}_i$ and $m_o$, respectively.
\end{theorem}
\begin{proof}
It follows from Bayes' theorem, standard tail bounds for Gaussian distributions and the union bound that 
\begin{align*}
    P\left(\left|\bm{\mu}_{i,m}^{(k)}\!-\!\thet_i \right|\geq \gamma\bm{\sigma}_{i,m}^{(k)}\right)\leq 2(3N\!+\!1)\exp\left(-\frac{\gamma^2}{2}\right).
\end{align*}
The joint probability over all local estimators $m=1,\ldots, 3$ and all agents $i=1,\ldots,N$
can be obtained through a second application of the union bound, such that we obtain 
$\gamma=\sqrt{2\log\left(6N(3N+1)/\delta\right)}$. 
Next, we define the 
weight matrices $W_{i,m}^{(k)} = \Sigma_i^{(k)} (\Sigma_{i,m}^{(k)})\inv/3$.
It it is straightforward to see that $\sum_{m=1}^3W_{i,m}=I$ holds, such that the estimation error
after the first generalized product of experts aggregation can be bounded with probability of at least $1-\delta$ 
by
\begin{align*}
    |\bm{\mu}_{i}^{(k)} -\thet_i |
    &\leq \sum_{m=1}^3\gamma |W_{i,m}^{(k)}|\bm{\sigma}_{i,m}^{(k)}=\bm{\eta}_i^{(k)}.
\end{align*}
Due to the 
different treatment of estimates in the following processing steps, a case distinction is necessary. 
In the first case, we consider the error bound for $\tilde{\mu}_{i,3N+1}^{(k)}$. Since this value is 
computed using the distributed generalized product of experts aggregation 
of $\mu_{i,3N+1}^{(k)}$, the error bounds are aggregated
similarly as within each agent, with the slight difference of an additional consensus error $\kappa_{i,3N+1}^{(k)}$, 
whose specific expression is postponed for the moment. This yields
$|\tilde{\mu}_{i,3N+1}^{(k)}-\vartheta_{3N+1}|\leq \hat{\eta}_{3N+1}^{(k)}+\kappa_{i,3N+1}^{(k)}$.
For all other entries of $\tilde{\bm{\mu}}_i^{(k)}$, there is exactly one agent computing the
estimate used in the distributed consensus based on \eqref{eq:ratiomean}, while all other agents
compute the estimate through \eqref{eq:divmean}. Since the error bound defines a confidence interval,
the quotient between \eqref{eq:ratiomean} and the corresponding entry of $\bm{\vartheta}$ 
can be straightforwardly bounded using interval 
arithmetic~\cite{Alefeld2000}, which results in \eqref{eq:rho} due to \eqref{eq:divCond}. The 
estimate in \eqref{eq:divmean} is obtained by adding the result of \eqref{eq:ratiomean} to 
the corresponding entry of $\bm{\mu}_i^{(k)}$. Thus, the error bound for \eqref{eq:divmean}
is obtained by adding the individual error bounds. Due to the distributed generalized product
of expert aggregation, we have in the limit $k\rightarrow\infty$ that
\begin{align*}
    \lim_{k\rightarrow\infty} P\left(|\tilde{\bm{\mu}}_i^{(k)}\!-\!\bm{\vartheta}|\leq\hat{\bm{\eta}}^{(k)}\!+\!\bm{\kappa}_i^{(k)}\!,~\forall i=1,\ldots,N\right)\geq 1\!-\!\delta
\end{align*}
such that it remains to prove that the consensus error $\bm{\kappa}_i^{(k)}$ asymptotically vanishes.
In order to show this, we make use of \cite[Corollary 3.1]{Zhu2010}, which is applicable due to 
Assumption~\ref{as:connectivity}, 
and requires the difference between the inputs 
to the dynamic average consensus to vanish, i.e., $\bm{\psi}_i^{(k+1)}-\bm{\psi}_i^{(k)}\overset{!}{\rightarrow} \bm{0}$.
This holds in the considered approach, since the posterior variance of Bayesian regression is 
non-increasing with respect to k and bounded from below by~$0$, such that the monotone convergence theorem guarantees
that $\Sigma_{i,m}^{(k)}\rightarrow \bar{\Sigma}$ for some $\bar{\Sigma}\in\mathbb{R}^{(3N+1)\times (3N+1)}$.
This directly implies convergence of $\bm{\mu}_{i,m}^{(k)}$, and since $1/\hat{\bm\sigma}_{i}^{(k)}$ is
assumed to not diverge, 
$\bm{\psi}_i^{(k+1)}-\bm{\psi}_i^{(k)}\rightarrow \bm{0}$ holds. 
\end{proof}

\subsection{Rotational Regression Problem}\label{sssec:rot_model}
Based on the results of the translational regression task, in this section we derive the estimator for the rotational degrees of freedom. Due to limited space and similarity in the derivation, we will omit most of the details and focus on the differences.
Following similar steps as in~\sref{sssec:trans_model} we obtain the rotational model for each agent $i$ as
\begin{align}\label{eq:model_rot}
\bm y^r_i &= \bm \phi^{r{\mkern-1.5mu\mathsf{T}}}_i\thet^r\\
    \bm y^r_i(\z) &= - \sumaj \left[S(\rn_j)\left( m_j\Delta\ddp_j+d_j\Delta\dpn_j+k_j\Delta\p_j +\bm f_j^d\right) \right.\nonumber\\ 
	&\left.+ J_j\Delta\fpkt{\omega}_j+\delta_j\Delta\fett{\omega}_j+\kappa\strich_j\Delta\fett{\epsilon}_j + \fett{t}_j^d\right]\nonumber\\
	\bm \phi^{r{\mkern-1.5mu\mathsf{T}}}(\bm{z}_i) &= \bm V(\R{o})\left(\dotm{\dw_i} + \Skew{\w_i}\dotm{\w_i}\right)\nonumber\\
	\thet^r &= \bmat{{}^oJ_o^{11} & {}^oJ_o^{12} & {}^oJ_o^{13} & {}^oJ_o^{22} & {}^oJ_o^{23} & {}^oJ_o^{33}}\tran\nonumber\\
	\dotm{\dw}&=\bmat{\dw^1 & \dw^2 & \dw^3 & 0 & 0 & 0\\ 0 & \dw^1 & 0 & \dw^2 & \dw^3 &0\\ 0 & 0 & \dw^1 & 0 & \dw_2 & \dw_3}\nonumber
\end{align}
where $\bm V$ is the matrix such that $\bmat{J_o^{11} & J_o^{12} & J_o^{13} & J_o^{22} & J_o^{23} & J_o^{33}} = \bm V \thet_r$.
Note that for the model~\eqref{eq:model_rot} information about the unknown parameters $\rn_j$ and states $\z_j$ of all agents is required. While it is possible to reformulate the equations using the techniques presented for the translational estimator, this would lead to a high number of unknown parameters for the regression task. In this work we follow a different approach by noticing that estimates $\ori$ and ${}^o\rn_{i,j}$ exist from the first estimator. As a result, by applying the transformations~~\eqref{eq:pos_rel},~\eqref{eq:vel_rel},~\eqref{eq:acc_rel}, agent $i$ can obtain an estimate of the states $\z_j$ of agent $j$ for $i\neq j$.
With this, we can reformulate
\begin{align}
    &\bm y_i^r(\z_i,\tilde{\bm \mu}_i^{(k)}) = - \sumaj \left[S(\widehat{\rn}_j)\left( m_j\widehat{\Delta\ddp}_j+d_j\widehat{\Delta\dpn}_j+k_j\widehat{\Delta\p}_j\right) \right.\nonumber \\ 
	&\left.+ J_j(\fpkt{\omega}_i-\dw_j^d)+\delta_j(\fett{\omega}_i-\w_j^d)+\kappa\strich_j\Delta\fett{\epsilon}_j + S(\widehat{\rn}_j)\bm f_j^d + \fett{t}_j^d\right].
\end{align}
Then, following the same procedure as for the translational estimator, we can use this model to perform Bayesian linear regression, aggregate the individual estimates using the generalized product of experts, and perform dynamic average consensus on the resulting estimates. Note that there is no decomposition step required since the parameters $\thet^r$ resemble the desired parameters, i.e., $\thet^r=\bm{\vartheta}^r$. \cref{th:thetabound} can be straightforwardly adapted to this procedure by bounding the additional error caused by using $\bm y_i^r(\z_i,\tilde{\bm \mu}_i^{(k)})$ as training target instead of the true value $\bm y_i^r(\z_i,\bm{\vartheta})$. Since the function $\bm y_i^r(\z_i,\cdot)$ is Lipschitz continuous, the difference between these two targets is bounded by $\|\bm y_i^r(\z_i,\tilde{\bm \mu}_i^{(k)})-\bm y_i^r(\z_i,\bm{\vartheta})\|\leq L_{y^r}\|\hat{\bm{\eta}}^{(k)}\|$. This can directly be propagated through the Bayesian linear regression, i.e., the analogue of \eqref{eq:mean} for the rotational regression task. Thereby, the error caused by the targets $\bm y_i^r(\z_i,\tilde{\bm \mu}_i^{(k)})$ in the regression can be bounded by
$\beta L_{y^r} \|\bm{\phi}(\bm{z}^{(k)})\|\|\Sigma_{i,m}^r\| \|\hat{\bm{\eta}}^k\|$.

\section{NUMERICAL EVALUATION}
\label{sec:evaluation}
In this section we evaluate the proposed learning framework in a simulated cooperative manipulation task, where four agents cooperatively manipulate a hollow sphere. 
\subsection{Simulation Setup}
The agent dynamics in~\eqref{eq:impedance} are chosen with homogeneous parameters $m_i = 1,\J_i = 0.5\I_3, d_i = 150, \delta_i = 1, k_i = 100, \kappa_i = 0.15$ for each agent $i$. The agents cooperatively grasp and manipulate a hollow sphere with radius $r_o = 0.325$ and dynamics~\eqref{eq:object_dynamics}, where $m_o = 10, J_o = \frac{2}{3}m_or_o^2$. The gravitational acceleration is approximated as $g = 9.81$. The initial values for the estimator are drawn from a normal distribution as $\bm \mu_{i,m}^{(0)}\sim\mathcal{N}(\thet_{i},\varsigma_{i}^2)$, where $\varsigma = \bmat{0.5 \bm 1_9^T & 5\bm 1_3^T & 10 & \bm 1_6^T}$. The initial uncertainty is chosen as $\Sigma_{i,m}^{(0)} = 0.5\I$.
The communication graph is given as a circular graph with adjacency matrix\looseness=-1
\begin{equation}
	\bm A = \bmat{1/3 & 1/3 & 0 & 1/3\\ 1/3 & 1/3 & 1/3 & 0\\ 0 & 1/3 & 1/3 & 1/3\\ 1/3 & 0 & 1/3 & 1/3}.
\end{equation}
The outputs~$\fett{y}_i$ are corrupted by additive white noise with variance~$\beta\inv = 2$, which is also the parameter chosen in the translational estimator. For the rotational estimator we choose $\beta_r << \beta$ during the first second and afterwards $\beta_r = \beta$. This is done such that only after the estimates $\thet$ have converged, the estimation of $\thet_r$ starts. 
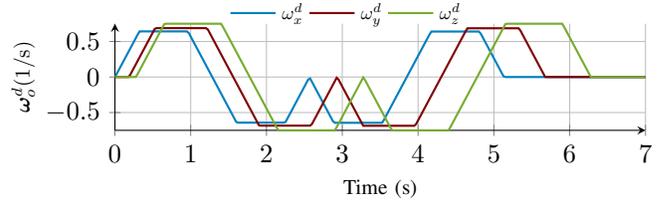
\begin{figure}
	\centering
	\begin{tikzpicture}
	\begin{axis}[
	clip = false,
	height=3cm,width=\columnwidth,
	grid = both,
	axis lines=left,
	ylabel near ticks,
	ylabel style={yshift = -0.1cm},
	ylabel={\labelsize{$\w_o^d (1/\si{s})$}},
	xlabel={\labelsize{Time (s)}},
	legend columns=-1,
	legend style={draw=none, fill=none,at={(.2,1.1)},anchor=west}
	]
	\addplot[color=blue,thick] table [x={time},y={dxd1}, col sep=comma]{./data/dxd.csv};
	\addplot[color=red,thick] table [x={time},y={dxd2}, col sep=comma]{./data/dxd.csv};
	\addplot[color=green,thick] table [x={time},y={dxd3}, col sep=comma]{./data/dxd.csv};
	\legend{\legendsize{$\omega^d_{x}$},\legendsize{$\omega^d_{y}$},\legendsize{$\omega^d_{z}$}};	
		
	\end{axis}
	\end{tikzpicture}
	\vspace{-4ex}
	\caption{Desired angular object velocity $\w_o^d$ around all axis ($x,y,z$).}
	\label{fig:dxd}
\end{figure}
The excitation signal consists of purely rotational desired object velocities as depicted in~\fref{fig:dxd}. The individual desired endeffector motion is then obtained via~\eqref{eq:vel_trans}. Since for this transformation the actual values of $\ori$ are required, which are unknown, they are drawn from a normal distribution~$\sim\mathcal{N}(\ori,0.01)$ for the purpose of simulation.

\subsection{Simulation Results}

\usepgfplotslibrary{patchplots,fillbetween}
\pgfplotsset{every tick label/.append style={font=\scriptsize}}
\pgfplotsset{every x tick label/.append style={font=\scriptsize, yshift=0.5ex}}
\pgfplotsset{every y tick label/.append style={font=\scriptsize, xshift=0.5ex}}

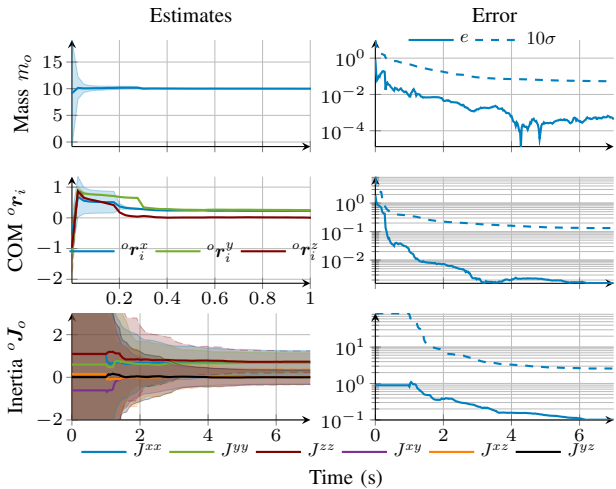
\begin{figure}[t]
\vspace{0.2cm}
	\centering
	\begin{tikzpicture}
		\begin{groupplot}[group style={group size= 2 by 3,horizontal sep = .1\columnwidth,vertical sep = .4cm},
			height=3cm,width=.55\columnwidth,
			grid = both,
			axis lines=left,
			ylabel near ticks,
			ylabel style={yshift = -0.1cm},]				
			
			\nextgroupplot[
			ylabel={\labelsize{Mass $m_o$}},
			xlabel={\labelsize{Estimates}},
			xlabel style={yshift = 2.5cm},
 			xticklabels={,,},
			xmax = 1,
			]
			\addplot+[name path=varp1, color=blue,opacity=0.3, no marks] table [x index=0,y expr=\thisrowno{1}+\thisrowno{2}]{./data/mass.dat};
			\addplot+[name path=varm1, color=blue,opacity=0.3, no marks] table [x index=0,y expr=\thisrowno{1}-\thisrowno{2}]{./data/mass.dat};
			\addplot[blue,opacity=0.2] fill between[ of = varm1 and varp1]; 
			\addplot[color=blue,thick] table [x index=0,y index=1]{./data/mass.dat};
			
						\nextgroupplot[
xticklabels={,,},
xlabel={\labelsize{Error}},
xlabel style={yshift = 2.5cm},
			legend columns=-1,
			legend style={draw=none, fill=none,at={(.1,1)},anchor=west},
ymode=log,   
]
\addplot[color=blue,thick] table [x index=0,y index=1]{./data/mass_log.dat};
\addplot[color=blue,thick, dashed] table [x index=0,y index=2]{./data/mass_log.dat};
			\legend{\legendsize{$e$},\legendsize{$10\sigma$}};	
			
			\nextgroupplot[
			ylabel={\labelsize{\ac{com} $\ori$}},
			legend columns=-1,
			legend style={draw=none, fill=none,at={(-.05,.3)},anchor=west},
 			xmax = 1,
			]
			\addplot+[name path=varp1, color=blue,opacity=0.3, no marks] table [x index=0,y expr=\thisrowno{1}+\thisrowno{4}]{./data/r1.dat};
			\addplot+[name path=varm1, color=blue,opacity=0.3, no marks] table [x index=0,y expr=\thisrowno{1}-\thisrowno{4}]{./data/r1.dat};
			\addplot[blue,opacity=0.2] fill between[ of = varm1 and varp1]; 
			\addplot[color=blue,thick] table [x index=0,y index=1]{./data/r1.dat};
			
			\addplot+[name path=varp2, color=green,opacity=0.3, no marks] table [x index=0,y expr=\thisrowno{2}+\thisrowno{5}]{./data/r1.dat};
			\addplot+[name path=varm2, color=green,opacity=0.3, no marks] table [x index=0,y expr=\thisrowno{2}-\thisrowno{5}]{./data/r1.dat};
			\addplot[green,opacity=0.2] fill between[ of = varm2 and varp2]; 
			\addplot[color=green,thick] table [x index=0,y index=2]{./data/r1.dat};
			
			\addplot+[name path=varp3, color=red,opacity=0.3, no marks] table [x index=0,y expr=\thisrowno{3}+\thisrowno{6}]{./data/r1.dat};
			\addplot+[name path=varm3, color=red,opacity=0.3, no marks] table [x index=0,y expr=\thisrowno{3}-\thisrowno{6}]{./data/r1.dat};
			\addplot[red,opacity=0.2] fill between[ of = varm3 and varp3]; 
			\addplot[color=red,thick] table [x index=0,y index=3]{./data/r1.dat};
			
			\legend{,,,\legendsize{$\ori^x$},,,, \legendsize{$\ori^y$},,,, \legendsize{$\ori^z$}};	
			
						\nextgroupplot[
			xticklabels={,,},
			ymode=log,    
			]
			\addplot[color=blue,thick] table [x index=0,y index=1]{./data/r1_log.dat};
			\addplot[color=blue,thick,dashed] table [x index=0,y index=2]{./data/r1_log.dat};

			\nextgroupplot[
			ylabel={\labelsize{Inertia ${}^o\J_o$}},
			xlabel={\labelsize{Time (s)}},
			xlabel style={at={(1.15,-.005)},anchor=north},
			legend columns=-1,
			legend style={draw=none, fill=none,at={(0,-.3)},anchor=west},
			ymax = 3,ymin = -2,
			]			
			\addplot+[name path=varp1, color=blue,opacity=0.3, no marks] table [x index=0,y expr=\thisrowno{1}+\thisrowno{7}]{./data/J.dat};
			\addplot+[name path=varm1, color=blue,opacity=0.3, no marks] table [x index=0,y expr=\thisrowno{1}-\thisrowno{7}]{./data/J.dat};
			\addplot[blue,opacity=0.2] fill between[ of = varm1 and varp1]; 
			\addplot[color=blue,thick] table [x index=0,y index=1]{./data/J.dat};
			
			\addplot+[name path=varp2, color=green,opacity=0.3, no marks] table [x index=0,y expr=\thisrowno{4}+\thisrowno{10}]{./data/J.dat};
			\addplot+[name path=varm2, color=green,opacity=0.3, no marks] table [x index=0,y expr=\thisrowno{4}-\thisrowno{10}]{./data/J.dat};
			\addplot[green,opacity=0.2] fill between[ of = varm2 and varp2]; 
			\addplot[color=green,thick] table [x index=0,y index=4]{./data/J.dat};
			
			\addplot+[name path=varp3, color=red,opacity=0.3, no marks] table [x index=0,y expr=\thisrowno{6}+\thisrowno{12}]{./data/J.dat};
			\addplot+[name path=varm3, color=red,opacity=0.3, no marks] table [x index=0,y expr=\thisrowno{6}-\thisrowno{12}]{./data/J.dat};
			\addplot[red,opacity=0.2] fill between[ of = varm3 and varp3]; 
			\addplot[color=red,thick] table [x index=0,y index=6]{./data/J.dat};
			
			\addplot+[name path=varp2, color=mycolor1,opacity=0.3, no marks] table [x index=0,y expr=\thisrowno{2}+\thisrowno{8}]{./data/J.dat};
			\addplot+[name path=varm2, color=mycolor1,opacity=0.3, no marks] table [x index=0,y expr=\thisrowno{2}-\thisrowno{8}]{./data/J.dat};
			\addplot[mycolor1,opacity=0.2] fill between[ of = varm2 and varp2]; 
			\addplot[color=mycolor1,thick] table [x index=0,y index=2]{./data/J.dat};
			
			\addplot+[name path=varp3, color=mycolor2,opacity=0.3, no marks] table [x index=0,y expr=\thisrowno{3}+\thisrowno{9}]{./data/J.dat};
			\addplot+[name path=varm3, color=mycolor2,opacity=0.3, no marks] table [x index=0,y expr=\thisrowno{3}-\thisrowno{9}]{./data/J.dat};
			\addplot[mycolor2,opacity=0.2] fill between[ of = varm3 and varp3]; 
			\addplot[color=mycolor2,thick] table [x index=0,y index=3]{./data/J.dat};
			
			\addplot+[name path=varp1, color=mycolor3,opacity=0.3, no marks] table [x index=0,y expr=\thisrowno{5}+\thisrowno{11}]{./data/J.dat};
			\addplot+[name path=varm1, color=mycolor3,opacity=0.3, no marks] table [x index=0,y expr=\thisrowno{5}-\thisrowno{11}]{./data/J.dat};
			\addplot[mycolor3,opacity=0.2] fill between[ of = varm1 and varp1]; 
			\addplot[color=mycolor3,thick] table [x index=0,y index=5]{./data/J.dat};
			\legend{,,,\legendsize{$J^{xx}$},,,,\legendsize{$J^{yy}$},,,, \legendsize{$J^{zz}$},,,, \legendsize{$J^{xy}$},,,, \legendsize{$J^{xz}$},,,, \legendsize{$J^{yz}$}};

						\nextgroupplot[
ymode=log,    
]
\addplot[color=blue,thick] table [x index=0,y index=1]{./data/J_log.dat};
\addplot[color=blue,thick,dashed] table [x index=0,y index=2]{./data/J_log.dat};

		\end{groupplot}
	\end{tikzpicture}
	\vspace{-1ex}
	\caption{Estimates (left) and estimation errors (right) including the uncertainty plotted as $10\sigma$ as shaded area. From top to bottom: mass $m_j$, \ac{com} $\ori$, inertia $J_o$.}
	\label{fig:eval}
\end{figure}
Due to space constraints, we only discuss the results for the first agent and note that the remaining agents yield similar results. 
The estimates including uncertainties are depicted in the left of~\fref{fig:eval}, while the estimation error is presented in the right of~\fref{fig:eval}. The estimation error is defined as the Euclidean norm of the difference between true and estimated value $e_l = ||\tilde{\bm{\mu}}_l-\bm{\vartheta}_l||$, where an index $l\in \{m,r,J\}$ denotes the entries corresponding to the parameters $m_o,\ori,{}^oJ_o$, respectively. In the upper two plots it can be seen that the estimates for the mass $m_o$ and the \ac{com} $\ori$ converge towards the true values within a few time-steps, while the uncertainty decreases. After one second, the estimation errors for $m_o$ are given as $e_{m} = 0.008$, which corresponds to a relative error of less then $1\%$, and $e_{r} \approx 0.02$ for the \ac{com}, corresponding to a relative error of around $6\%$. After one second the estimator for the inertia is activated and the estimation error decreases, as shown in in the bottom plots of~\fref{fig:eval}. After the simulation time of seven seconds an error of $e_{J} \approx 0.1$ for the inertia remains, which corresponds to roughly $8\%$ of the actual value. Note that the estimate is accompanied with a relatively high uncertainty, such that for critical tasks control parameters could be updated to account for the uncertainty~\cite{Beckers2019}, which is not possible with standard approaches, which do not provide this measure of uncertainty.\looseness=-1


\section{CONCLUSION}
\label{sec:conclusion}
In this paper we present a novel distributed online learning framework for cooperative manipulation using Bayesian principles. We derive a generalized linear model, which only requires locally available information. With this model the parameters are identified with Bayesian linear regression and combined with dynamic average consensus to obtain a common estimate. This allows us to provide a bound for the prediction error with high probability and iterative learning with constant complexity, making it suitable for online learning. The approach is illustrated in a simulated cooperative manipulation setting.


\bibliographystyle{IEEEtran}
\bibliography{references}

\begin{thebibliography}{10}
\providecommand{\url}[1]{#1}
\csname url@samestyle\endcsname
\providecommand{\newblock}{\relax}
\providecommand{\bibinfo}[2]{#2}
\providecommand{\BIBentrySTDinterwordspacing}{\spaceskip=0pt\relax}
\providecommand{\BIBentryALTinterwordstretchfactor}{4}
\providecommand{\BIBentryALTinterwordspacing}{\spaceskip=\fontdimen2\font plus
\BIBentryALTinterwordstretchfactor\fontdimen3\font minus
  \fontdimen4\font\relax}
\providecommand{\BIBforeignlanguage}[2]{{%
\expandafter\ifx\csname l@#1\endcsname\relax
\typeout{** WARNING: IEEEtran.bst: No hyphenation pattern has been}%
\typeout{** loaded for the language `#1'. Using the pattern for}%
\typeout{** the default language instead.}%
\else
\language=\csname l@#1\endcsname
\fi
#2}}
\providecommand{\BIBdecl}{\relax}
\BIBdecl

\bibitem{Dohmann.2020}
P.~B.~g. Dohmann and S.~Hirche, ``{Distributed Control for Cooperative
  Manipulation With Event-Triggered Communication},'' \emph{{IEEE Transactions
  on Robotics}}, vol.~36, no.~4, pp. 1038--1052, 2020.

\bibitem{Marino.2018.2}
A.~Marino, ``{Distributed Adaptive Control of Networked Cooperative Mobile
  Manipulators},'' \emph{{IEEE Transactions on Control Systems Technology}},
  vol.~26, no.~5, pp. 1646--1660, 2018.

\bibitem{Erhart.2016}
S.~Erhart and S.~Hirche, ``{Model and Analysis of the Interaction Dynamics in
  Cooperative Manipulation Tasks},'' \emph{{IEEE Transactions on Robotics}},
  vol.~32, no.~3, pp. 672--683, 2016.

\bibitem{Cehajic.2017}
D.~Cehajic, P.~B.~g. Dohmann, and S.~Hirche, ``{Estimating unknown object
  dynamics in human-robot manipulation tasks},'' in \emph{{IEEE International
  Conference on Robotics and Automation (ICRA)}}, 2017.

\bibitem{Pierri.2020}
F.~Pierri, M.~Nigro, G.~Muscio, and F.~Caccavale, ``{Cooperative Manipulation
  of an Unknown Object via Omnidirectional Unmanned Aerial Vehicles},''
  \emph{{Journal of Intelligent {\&} Robotic Systems}}, vol. 100, no. 3-4, pp.
  1635--1649, 2020.

\bibitem{Marino.2017}
A.~Marino, G.~Muscio, and F.~Pierri, ``{Distributed cooperative object
  parameter estimation and manipulation without explicit communication},'' in
  \emph{{IEEE Int. Conf. on Robotics and Automation (ICRA)}}, 2017.

\bibitem{Franchi.2015}
A.~Franchi, A.~Petitti, and A.~Rizzo, ``{Decentralized parameter estimation and
  observation for cooperative mobile manipulation of an unknown load using
  noisy measurements},'' in \emph{{IEEE International Conference on Robotics
  and Automation (ICRA), 2015}}, 2015.

\bibitem{Bishop2006}
C.~M. Bishop, \emph{{Pattern Recognition and Machine Learning}}.\hskip 1em plus
  0.5em minus 0.4em\relax New York, NY: Springer Science+Business Media, 2006.

\bibitem{Cao2014}
\BIBentryALTinterwordspacing
Y.~Cao and D.~J. Fleet, ``{Generalized Product of Experts for Automatic and
  Principled Fusion of Gaussian Process Predictions},'' pp. 1--5, 2014.
  [Online]. Available: \url{http://arxiv.org/abs/1410.7827}
\BIBentrySTDinterwordspacing

\bibitem{Deisenroth.2015}
M.~Deisenroth and J.~W. Ng, ``{Distributed Gaussian Processes},'' in
  \emph{Proceedings of the International Conference on Machine Learning},
  Lille, France, 2015, pp. 1481--1490.

\bibitem{Beckers2019}
T.~Beckers, D.~Kuli{\'{c}}, and S.~Hirche, ``{Stable Gaussian Process based
  Tracking Control of Euler-Lagrange Systems},'' \emph{Automatica}, vol. 103,
  no.~23, pp. 390--397, 2019.

\bibitem{Hewing2020}
L.~Hewing, J.~Kabzan, and M.~N. Zeilinger, ``{Cautious Model Predictive Control
  using Gaussian Process Regression},'' \emph{IEEE Transactions on Control
  Systems Technology}, vol.~28, no.~6, pp. 2736--2743, 2020.

\bibitem{Capone2020c}
A.~Capone, G.~Noske, J.~Umlauft, T.~Beckers, A.~Lederer, and S.~Hirche,
  ``{Localized active learning of Gaussian process state space models},'' in
  \emph{Learning for Dynamics {\&} Control}, 2020, pp. 1--10.

\bibitem{4639601}
F.~{Caccavale}, P.~{Chiacchio}, A.~{Marino}, and L.~{Villani}, ``Six-dof
  impedance control of dual-arm cooperative manipulators,'' \emph{IEEE/ASME
  Transactions on Mechatronics}, vol.~13, no.~5, pp. 576--586, 2008.

\bibitem{Tron2009}
R.~{Tron} and R.~{Vidal}, ``Distributed image-based 3-d localization of camera
  sensor networks,'' in \emph{Proceedings of the 48h IEEE Conference on
  Decision and Control (CDC) held jointly with 2009 28th Chinese Control
  Conference}, 2009, pp. 901--908.

\bibitem{Zhu2010}
M.~Zhu and S.~Mart{\'{i}}nez, ``{Discrete-time dynamic average consensus},''
  \emph{Automatica}, vol.~46, no.~2, pp. 322--329, 2010.

\bibitem{Fahrmeir.2009}
L.~Fahrmeir, T.~Kneib, and S.~Lang, ``{Regression: Modelle, Methoden und
  Anwendungen},'' Berlin, 2009.

\bibitem{Diaz-Frances2013}
E.~D{\'{i}}az-Franc{\'{e}}s and F.~J. Rubio, ``{On the Existence of a Normal
  Approximation to the Distribution of the Ratio of Two Independent Normal
  Random Variables},'' \emph{Statistical Papers}, vol.~54, no.~2, pp. 309--323,
  2013.

\bibitem{Chowdhury2017a}
S.~R. Chowdhury and A.~Gopalan, ``{On Kernelized Multi-armed Bandits},'' in
  \emph{Proceedings of the International Conference on Machine Learning}, 2017,
  pp. 844--853.

\bibitem{Maddalena2020}
\BIBentryALTinterwordspacing
E.~T. Maddalena, P.~Scharnhorst, and C.~N. Jones, ``{Deterministic error bounds
  for kernel-based learning techniques under bounded noise},'' pp. 1--9, 2020.
  [Online]. Available: \url{https://arxiv.org/pdf/2008.04005}
\BIBentrySTDinterwordspacing

\bibitem{Alefeld2000}
G.~Alefeld and G.~Mayer, ``{Interval analysis: theory and applications},''
  \emph{Journal of Computational and Applied Mathematics}, vol. 121, pp.
  421--464, 2000.

\end{thebibliography}

\end{document}